\documentclass[11pt]{article}
\usepackage{fullpage}

\usepackage{amsfonts}
\usepackage{epsfig}
\usepackage{amsthm}
\usepackage{amsmath}
\usepackage{amssymb}
\usepackage{algorithm}
\usepackage{algorithmic}
\usepackage{subfig}
\usepackage{graphicx}
\usepackage{color}
\usepackage[section] {placeins}
\usepackage{epstopdf}
\usepackage{bbm}

\usepackage[section] {placeins}

\newtheorem{theorem}{Theorem}
\newtheorem{prop}{Proposition}
\newtheorem{cor}{Corollary}

\def\A{{\mathcal A}}

\def\D{{\mathbb D}}
\def\g{{\bf g}}
\def\I{{\bf I}}
\def\u{{\bf u}}
\def\w{{\bf w}}
\def\W{{W_{\mathrm{lam}}}}
\def\x{{\bf x}}
\def\z{{\bf z}}
\def\H{{\mathcal H}}

\def\X{{\mathcal X}}
\def\Y{{\mathcal Y}}
\def\btheta{\pmb \theta}
\def\R{\mathbb{R}}
\def\bdelta{\pmb \delta}

\def\Argmin{\ensuremath{{\mathrm{Argmin}}}}

\title{The Teaching Dimension of Linear Learners}

\author{Ji Liu$^\dag$ and Xiaojin Zhu$^\ddag$\\
\textit{jliu@cs.rochester.edu, jerryzhu@cs.wisc.edu}\\
$^\dag$Department of Computer Science, University of Rochester \\
$^\ddag$Department of Computer Sciences, University of Wisconsin-Madison
}
\date{\today}

\begin{document}
\maketitle
\thispagestyle{plain} \pagestyle{plain}

\begin{abstract}
Teaching dimension is a learning theoretic quantity that specifies the minimum training set size to teach a target model to a learner. Previous studies on teaching dimension focused on version-space learners which maintain all hypotheses consistent with the training data, and cannot be applied to modern machine learners which select a specific hypothesis via optimization. This paper presents the first known teaching dimension for ridge regression, support vector machines, and logistic regression.  We also exhibit optimal training sets that match these teaching dimensions. 
Our approach generalizes to other linear learners.

\end{abstract}

\section{Introduction}
\label{sec:intro}

Consider a teacher who knows both a target model and the learning algorithm used by a machine learner.
The teacher wants to teach the target model to the learner by \emph{constructing} a training set. 
The training set does not need to contain independent and identically distributed items drawn from some distribution.
Furthermore, the teacher can construct any item in the input space.
How many training items are needed?
This is the question addressed by the \emph{teaching dimension}~\cite{Goldman1995Complexity,Shinohara1991Teachability}. 
We give the precise definition in section~\ref{sec:classicTD}, but first illustrate the intuition with an example.

Consider integers $x \in \{ 1 \ldots 10 \}$ and threshold classifiers $h_\theta$ on them, so that $h_\theta(x)$ returns -1 if $x < \theta$ and 1 if $x \ge \theta$.
Now let the hypothesis space $\H$ consist of eleven classifiers $\H=\{h_\theta \mid \theta \in \{1 \ldots 11\}\}$. 
Let the learner be a version-space learner, namely it maintains a version space $\{h_\theta \in \H \mid h_\theta \mbox{ consistent with the training set}\}$.
If we want to teach a target model (in this paper we use hypothesis and model exchangeably), say $h_9$, to such a learner,
we can construct a training set that results in a singleton version space $\{h_9\}$.
It is easy to see that the training set $D=\{(x_1=8,y_1=-1), (x_2=9,y_2=1)\}$ is the smallest set for this purpose. 
We say that the teaching dimension of $h_9$ with respect to $\H$ is $TD(h_9) = |D|=2$.
Similarly, $TD(h_{11})=1$ because $D=\{(x_1=10, y_1=-1)\}$ suffices. In fact, $TD(h_\theta^*)=1$ for target model $\theta^*=1$ or 11, and 2 for $\theta^*=2,3,\ldots,10$.

The astute reader may notice that this example does not apply to continuous spaces.  To see this, let us extend $x \in \R$ and $\H=\{h_\theta \mid \theta \in \R\}$. The learner's version space under any linearly separable training set would now be represented by the interval between the two closest oppositely labeled items. 
It is impossible for the version-space learner to pick out a unique target model $h_{\theta^*}$ with a finite training set. In other words, $TD(h_{\theta^*})=\infty$ for all target models $\theta^*$. This is counterintuitive because ostensibly we can teach any one of the ``modern'' machine learning algorithms such as a support vector machine (SVM) with only two training items: $D=\{ (x_1=\theta^*-\epsilon, y_1=-1), (x_2=\theta^*+\epsilon, y_2=1) \}$ with any $\epsilon > 0$.

The issue here is that a version-space learner is not equipped with the ability to pick the max-margin (or any other specific) hypothesis from the version space. In contrast, an SVM is \emph{not} a version-space learner in our terminology; 
we have stronger knowledge from optimization on how it picks a specific hypothesis from the hypothesis space. 
This paper will utilize such knowledge to derive teaching dimensions that are distinct from classic teaching dimension analysis (e.g.~\cite{JMLR:v15:doliwa14a}). 
Specifically,
we extend teaching dimension to linear learners that learn by regularized empirical risk minimization:
\begin{equation}
\A_{opt}(D) := \Argmin_{\btheta \in \R^d} \quad \underbrace{\sum_{i=1}^n \ell(\x_i^\top {\btheta}, y_i) + {\lambda \over 2} \| \btheta \|_A^2}_{=:f(\btheta)}.
\label{eq:main}
\end{equation}
Here, we identify $\H$ with $\R^d$, $h$ with $\btheta$, 
the loss function $\ell$ is either smooth or convex in the first argument, 
$\lambda>0$ is the regularization coefficient, and $A$ is a positive semidefinite matrix.
$\|\cdot \|_A$ is the Mahalanobis norm: $\|\btheta\|_A:= \sqrt{\btheta^\top A \btheta}$.
We follow the convention in optimization when we use the capitalized {$\Argmin$} to emphasize that it returns a \emph{set} that achieves the minimum.
The teacher can construct a training set with any items in $\R^d$. 
The alternative pool-based teaching setting, where the teacher is given a finite pool of candidate training items and must select items from that pool,  is not studied in this paper.
By linear learners we mean the input $\x$ and the parameter $\btheta$ interact only via their inner product $\x^\top \btheta$.
Linear learners include SVMs, logistic regression, and linear regression.
Our analysis technique involves a novel application of the Karush-Kuhn-Tucker (KKT) conditions.

To our knowledge, this paper gives the first known values of teaching dimension for ridge regression, SVM, and logistic regression. 
We summarize our main results in Table~\ref{tab:main}. 
The table separately lists homogeneous (without a bias term) and inhomogeneous (with a bias term) versions of the linear learners.
The teaching goal refers to the intention of the teacher: is teaching considered successful only when the learner learns the exact target parameter, or when the learner learns the correct decision boundary (which can be achieved by any positive scaling of the target parameter).
See section~\ref{sec:main} for definition of the target parameters $\btheta^*, \w^*$ and the constant $\tau_{\mathrm{max}}$. 
The target parameters are assumed to be nonzero. 
We will also present the corresponding minimum teaching set construction in section~\ref{sec:main}.

\begin{table}[ht]
\begin{center}
\begin{tabular}{|c||c|c|c||c|c|c|}
\hline
& \multicolumn{3}{c||}{homogeneous} & \multicolumn{3}{c|}{inhomogeneous} \\ \cline{2-7}
& ridge & SVM & logistic & ridge & SVM & logistic \\
\hline
exact parameter 
	& 1 
	& $\left\lceil \lambda \|\btheta^*\|^2 \right\rceil$
	& $\left\lceil {\lambda \|\btheta^*\|^2 \over \tau_{\max}} \right\rceil$
	& 2 
	& $2 \left\lceil {{\lambda \|\w^*\|^2} \over 2} \right\rceil ^\dagger$
	& $2 \left\lceil { {\lambda \|\w^*\|^2} \over {2\tau_{\max}} } \right\rceil ^\dagger$ \\
\hline
decision boundary
	& -
	& 1
	& 1
	& -
	& 2
	& 2  \\
\hline
\end{tabular}
\end{center}
\caption{The teaching dimension of ridge regression, SVM, and logistic regression.  ($\dagger$: up to rounding effect, see section~\ref{sec:UBinhomogeneous}).}
\label{tab:main}
\end{table}

\section{Classic Teaching Dimension and its Limitations}
\label{sec:classicTD}

Let $\X$ denote the input space and $\Y \subseteq \R$ the output space.
A hypothesis is a function {$h: \X \rightarrow \Y$}.
In this paper we identify a hypothesis $h_{\btheta}$ with its model parameter $\btheta$.
The hypothesis space $\H$ is a set of hypotheses.
By training item we mean a pair $(x,y) \in \X \times \Y$.
A training set is a multiset $D=\{(\x_1, y_1) \ldots (\x_n, y_n)\}$ where repeated items are allowed.
Importantly, for the purpose of teaching we do \emph{not} assume that $D$ be drawn $i.i.d.$ from a distribution.
Let $\D$ denote the set of all training sets of all sizes.
A learning algorithm $\A: \D \rightarrow 2^\H$ takes in a training set $D \in \D$ and outputs a subset of the hypothesis space $\H$.
That is, $\A$ does not necessarily return a unique hypothesis.

Classic teaching dimension analysis is restricted to the version-space learner $\A_{vs}$:
\begin{equation}
\A_{vs}(D) = \{h \in \H \mid \mbox{ $h$ is consistent with $D$ } \}.
\label{eq:versionspace}
\end{equation}
That is, the learner $\A_{vs}$ keeps track of the version space.
Let the target model be $h_{\btheta^*} \in \H$.
Teaching is successful if the teacher identifies a training set $D \in \D$ such that {$\A_{vs}(D)=\{h_{\btheta^*}\}$} the singleton set.
Such a $D$ is called a \textbf{teaching set} of $h_{\btheta^*}$ with respect to $\H$.
The teaching dimension of the hypothesis $h_{\btheta^*}$ is the minimum size of the teaching set:
\begin{equation}
TD(h_{\btheta^*}) = \left\{
\begin{array}{cl}
\min_{D \in \D}  |D|, & \mbox{for $D$ a teaching set of $h_{\btheta^*}$} \\
\infty, & \mbox{if no teaching set exists}
\end{array}
\right.
\end{equation}
Furthermore, the teaching dimension of the whole hypothesis space $\H$ is defined by the hardest hypothesis: $TD(\H) = \max_{h \in \H} TD(h)$.
In this paper we will focus on the fine-grained teaching dimension of individual hypothesis $TD(h)$.

Classic teaching dimension analysis has several limitations:
the learner is assumed to be a version-space learner $\A_{vs}$,
and the hypothesis space is typically finite or countably infinite.
As the example in section~\ref{sec:intro} showed, these fail to capture the teaching dimension of ``modern'' machine learners which has $\R^d$ as input space and picks a unique hypothesis via regularized empirical risk minimization~\eqref{eq:main}.
Furthermore, the target model can be ambiguous when the learner is a classifier: should the learner learn the exact target parameter $\btheta^*$, or the target decision boundary?
In linear models any scaled parameter $c \btheta^*$ with $c>0$ produces the same target decision boundary.
These limitations motivate us to generalize the teaching dimension in the next section.

\section{Main Results}
\label{sec:main}

To make teaching dimension's dependency on the learning algorithm explicit, henceforth we write teaching dimension with two arguments as
\begin{equation}
TD(h^*, \A)
\end{equation}
where $h^* \in \H$ is the target model, and $\A: \D \rightarrow 2^\H$ is the learning algorithm which given a training set $D \in \D$ returns a set of hypotheses $A(D)$.
We define teaching dimension to be the size of the smallest training set $D$ such that $A(D) = \{h^*\}$, the singleton set containing the target model.
With this notation, the classic teaching dimension is $TD(h^*, \A_{vs})$ where $\A_{vs}$ is the version space learning algorithm~\eqref{eq:versionspace}.
In this paper we focus on $\A_{opt}$ in~\eqref{eq:main} instead, namely linear learners in $\R^d$.
Linear learners include many popular members such as both homogeneous (without a bias term) and inhomogeneous (with a bias term) versions of linear regression, SVM, and logistic regression.
In addition, the linear interaction between $\x$ and $\btheta$ makes the loss function subgradient easy to compute, though in principle our analysis technique is applicable to other optimization-based learners, too.
In this section our goal is to teach the exact parameter $\btheta^*$, consequently our teaching dimension of interest is 
\begin{equation}
TD(\btheta^*, \A_{opt}).
\end{equation}
Later in section~\ref{sec:Gdb} for classification we will teach the decision boundary instead.

How to reason about our teaching dimension $TD(\btheta^*, \A_{opt})$?
It is the size of the \emph{smallest} training set $D$ with which~\eqref{eq:main} has a unique solution $\btheta^*$.
Our strategy is to first establish a number of lower bounds $LB \le TD(\btheta^*, \A_{opt})$ by showing that any training set with which~\eqref{eq:main} has a unique solution $\btheta^*$ must have at least $LB$ items. 
Section~\ref{sec:LB} is devoted to such lower bounds. 
The actual teaching dimension is learner dependent.
In sections~\ref{sec:UBhomogeneous} 
and \ref{sec:UBinhomogeneous} 
we construct specific teaching sets for three popular learners: ridge regression, SVM, and logistic regression.
These teaching sets uniquely returns $\btheta^*$ via~\eqref{eq:main}.
By definition, the size of these teaching sets is an upper bound on $TD(\btheta^*, \A_{opt})$, respectively.
If the lower and upper bounds match, we would have identified the teaching dimension $TD(\btheta^*, \A_{opt})$.


\subsection{Lower Bounds on Teaching Dimension $TD(\btheta^*, \A_{opt})$}
\label{sec:LB}
In this section we provide three general lower bounds on the teaching dimension.
These lower bounds capture different aspects of a teaching set, and should be used in conjunction (i.e. taking the maximum) when applicable.
We will instantiate these lower bounds for specific learners in section~\ref{sec:UBhomogeneous}.
In the following let $\mathcal{X}$ and $\mathcal{Y}$ be the feasible region of all $\x_i$'s and $y_i$'s respectively. 
We will use the notation $\partial_1 \ell(\cdot, \cdot)$ in the following way: if $\ell(\cdot, \cdot)$ is smooth, then it denotes a singleton set only containing the gradient w.r.t. the first argument; if $\ell(\cdot, \cdot)$ is convex, then it denotes the subdifferential w.r.t the first argument. 

LB1 comes from a degree-of-freedom perspective.  It is necessary to have this amount of training items for a unique solution to exist in~\eqref{eq:main}.

\begin{theorem}
\label{thm:LB1}
Given any target model $\btheta^*$, there is a degree-of-freedom lower bound on the number of training items to obtain a unique solution $\btheta^*$ from solving \eqref{eq:main}: 
\begin{equation}
\label{eq:thm:1}
LB1 =
\begin{cases}
    d - \text{Rank}(A) + 1,& \text{if}~A\btheta^* \neq {\bf 0}\\
    d - \text{Rank}(A),              & \text{otherwise}.
\end{cases}
\end{equation}
\end{theorem}


\begin{proof}
Let $n^*$ be the minimal number of training items to ensure a unique solution $\btheta^*$. First consider the case $n^*=0$. It happens if and only if $\btheta^* = {\bf 0}$ and $\text{Rank}(A) = d$, which is a special case of $A\btheta^* = {\bf 0}$. Clearly, this case is consistent with LB1.
Next consider the case $n^* \geq 1$.   
Since $\btheta^*$ solves \eqref{eq:main}, the KKT condition holds: 
\begin{align}
-\lambda A\btheta^* \in \sum_{i=1}^{n^*}\partial_1 \ell(\x_i^\top \btheta^*, y_i) \x_i.
\label{eq:kkt}
\end{align} 
We seek all $\bdelta$ such that $\btheta^* + \bdelta$ satisfies
\begin{equation}
A (\btheta^*+\bdelta) = A \btheta^* \quad \text{and} \quad \x_i^{\top} (\btheta^* + \bdelta) = \x_i^\top \btheta^* \quad \forall i=1,\cdots, n^*,
\label{eq:cond}
\end{equation}
For any such $\bdelta$, simple algebra verifies that $\btheta^*+ t\bdelta$ satisfies the KKT condition~\eqref{eq:kkt} for any $t\in [0, 1]$. 
Consequently, $\btheta^* + \bdelta$ also solves the problem in \eqref{eq:main}. To see this, we consider two situations: 
\begin{itemize}
\item If the loss function $\ell(\cdot, \cdot)$ is convex in the first argument, the KKT condition is a sufficient optimality condition, which means that $\btheta^*+ \bdelta$ solves \eqref{eq:main}. 
\item If the loss function $\ell(\cdot, \cdot)$ is smooth (not necessary convex) in the first argument, we have $f(\btheta^*) = f(\btheta^* + \bdelta)$ by using the Taylor expansion (recall $f$ is defined in equation~\ref{eq:main}):
\begin{align*}
f(\btheta^* + \bdelta) = & f(\btheta^*) + \langle \nabla f(\btheta^* + t\bdelta), \bdelta \rangle \quad (\text{for some $t \in [0,1] $})
\\ = & 
f(\btheta^*) + \left\langle \sum_{i=1}^{n^*} \nabla_1 \ell(\x_i^\top (\btheta^*+ t\bdelta),~y_i)\x_i + \lambda A(\btheta^* + t\bdelta)),~\bdelta\right\rangle 
\\ = & 
f(\btheta^*) + \left\langle \underbrace{ \sum_{i=1}^{n^*} \nabla_1 \ell(\x_i^\top \btheta^*,~y_i)\x_i + \lambda A\btheta^*}_{\text{={\bf 0} due to the KKT condition \eqref{eq:kkt}}},~\bdelta\right\rangle  
\\ = &
f(\btheta^*).
\end{align*}  
\end{itemize}
Therefore, $\btheta^* + \bdelta$ also solves \eqref{eq:main}. However, the uniqueness of $\btheta^*$ requires $\bdelta = {\bf 0}$ to be the only value satisfying \eqref{eq:cond}. This is equivalent to say 
\begin{align}
\text{Null}(A) \cap \text{Null}(\text{Span}\{\x_1, \cdots, \x_{n^*}\}) = \{{\bf 0}\}. 
\label{eq:cond2}
\end{align}
It indicates that
\[
\text{Rank}(A) + \text{Dim}(\text{Span}\{\x_1, \cdots, \x_{n^*}\}) \geq d.
\]
From $n^*\geq \text{Dim}(\text{span}\{\x_1, \cdots, \x_{n^*}\})$, we have $n^* \geq d - \text{Rank}(A)$. 
We proved the general case for LB1. 

If we have $A\btheta^* \neq {\bf 0}$, we can further improve LB1. 
Let $\g^*=(g^*_1, \ldots, g^*_{n^*})^\top$ be the vector satisfying 
\begin{align}
-\lambda A \btheta^* = \sum_{i=1}^{n^*} g_i^* \x_i  \quad \text{and}  \quad g^*_i \in \partial_1 \ell(\x_i^\top \btheta^*, y_i) \quad \forall i=1, 2, \cdots, n^*.
\label{eq:cond0}
\end{align}
Since $\btheta^*$ satisfies the KKT condition, such vector $\g^*$ must exist. 
Applying  $A\btheta^* \neq {\bf 0}$ to \eqref{eq:cond0}, we have $\g^* \neq {\bf 0}$ and 
\begin{equation}
\text{Dim}\left(\text{Span}\{A_{.1}, A_{.2}, \cdots, A_{.d}\} \cap \text{Span}\{\x_1, \cdots, \x_{n^*} \}\right) \geq 1.
\label{eq:cond3}
\end{equation}
To satisfy \eqref{eq:cond2}, we must have
\[
d = \text{Dim}\left(\text{Span}\{A_{.1}, A_{.2}, \cdots, A_{.d}, \x_1, \cdots, \x_{n^*} \}\right).
\]
Using the fact in linear algebra
\begin{align*}
\text{Dim}\left(\text{Span}\{A_{.1}, A_{.2}, \cdots, A_{.d}, \x_1, \cdots, \x_{n^*} \}\right) = & \underbrace{\text{Dim}\left(\text{Span}\{A_{.1}, A_{.2}, \cdots, A_{.d}\}\right)}_{=\text{Rank}(A)} + 
\\ &
\underbrace{\text{Dim}\left(\text{Span}\{\x_1, \cdots, \x_{n^*} \}\right)}_{\leq n^*} - 
\\ &
\underbrace{\text{Dim}\left(\text{Span}\{A_{.1}, A_{.2}, \cdots, A_{.d}\} \cap \text{Span}\{\x_1, \cdots, \x_{n^*} \}\right)}_{\geq 1~(\text{from \eqref{eq:cond3}})}
\end{align*}
We conclude that $n^* \geq d - \text{Rank}(A) + 1$. We completed the proof for LB1.
\end{proof}

LB2 observes that the regularizer acts as a prior.  If $\lambda$ is large, more items are needed to sway the prior toward the target $\btheta^*$.


\begin{theorem}
\label{thm:LB2}
Given any target model $\btheta^*$, 
there is a strength-of-regularization lower bound on the required number of training items to obtain a unique solution $\btheta^*$ from solving \eqref{eq:main}: 
\begin{equation}
\label{eq:thm:3}
LB2 = 
\begin{cases}
\left\lceil {\lambda}{\left(\sup_{\alpha \in \R, y\in \mathcal{Y}, g\in -\partial_1 \ell(\alpha \|\btheta^*\|^2_A, y)} \alpha g\right)^{-1}} \right\rceil, & \text{if $A$ has full rank and $\btheta^* \neq 0$} \\
0, & \text{otherwise.}
\end{cases}
\end{equation} 
\end{theorem}


\begin{proof}
When $A$ has full rank we have an equivalent expression for the KKT condition \eqref{eq:kkt}: 
\begin{align}
-\lambda A^{1\over 2} \btheta^* \in \sum_{i=1}^{n^*} A^{-{1\over 2}} \x_i \partial_1 \ell(\x_i^\top \btheta^*, y_i)\quad \forall i = 1,\cdots, n^*.
\label{eq:kkt_2}
\end{align}
Let us decompose $A^{-{1\over 2}} \x_i$ for all $i=1,\cdots, n^*$ into $A^{-{1\over 2}} \x_i = \alpha_i A^{1\over 2}\btheta^* + \u_i$, where $\u_i$ is orthogonal to $A^{1\over 2} \btheta^*$: $\u_i^\top A^{1\over 2} \btheta^* = 0$. 
Equivalently $\x_i = \alpha_i A \btheta^* + A^{1\over 2}\u_i$.
Applying this decomposition, we have
\[
\x_i^\top \btheta^* = \alpha_i \|\btheta^*\|^2_A + \u_i^\top A^{1 \over 2} \btheta^* = \alpha_i \|\btheta^*\|^2_A.
\]
Putting it back in \eqref{eq:kkt_2} we obtain
\begin{align}
-\lambda A^{1\over 2} \btheta^* \in & \sum_{i=1}^{n^*}\left(\alpha_i A^{{1\over 2}}\btheta^* +\u_i\right)\partial_1\ell(\alpha_i \|\btheta^*\|^2_A, y_i) 
 \quad \forall i = 1,\cdots, n^*.
\label{eq:kkt_3}
\end{align} 
Since $\u_i$ is orthogonal to $A^{1\over 2}\btheta^*$, \eqref{eq:kkt_3} can be rewritten as
\begin{align}
\label{eq:thm_proof_1}
&\exists \alpha_i \in \R,~\exists y_i\in \mathcal{Y},~\exists g_i \in \partial_1 \ell(\alpha_i\|\btheta^*\|^2_A, y_i)\quad \forall i=1, \cdots, n^* \\
\nonumber
\text{satisfying}\quad &\sum_{i=1}^{n^*} g_i\u_i = 0
\\ &
-\lambda A^{1\over 2}\btheta^* = A^{1\over 2}\btheta^* \sum_{i=1}^{n^*}\alpha_ig_i
\label{eq:thm_proof_3}
\end{align}
Since $A\btheta^* \neq 0$, we have $A^{1\over 2}\btheta^* \neq 0$ and \eqref{eq:thm_proof_3} is equivalent to $-\lambda = \sum_{i=1}^{n^*} \alpha_i g_i$. 
It follows that
\[
\lambda = -\sum_{i=1}^{n^*} \alpha_i g_i \leq n^* \sup_{\alpha\in\R, y\in \mathcal{Y}, g\in \partial_1 \ell(\alpha \|\btheta^*\|^2_A, y)} -\alpha g = n^* \sup_{\alpha\in\R, y\in \mathcal{Y}, g\in -\partial_1 \ell(\alpha \|\btheta^*\|^2_A, y)} \alpha g
\]
It indicates the lower bound for $n^*$
\[
n^* \geq \left\lceil \frac{\lambda}{\sup_{\alpha\in\R, y\in \mathcal{Y}, g\in -\partial_1 \ell(\alpha \|\btheta^*\|^2_A, y)} \alpha g} \right\rceil.
\]
\end{proof}

LB1 and LB2 apply to all generalized linear learners.
Due to the popularity of inhomogeneous margin-based linear learners (which include the standard form of SVM and logistic regression), we provide a tighter lower bound LB3 for such learners in Theorem~\ref{thm:classification}. 
For inhomogeneous margin-based linear learners the learning algorithm $\A_{opt}$ solves a special form of \eqref{eq:main}: 
\begin{equation}
\A_{opt}(D) = \Argmin_{\w, b}\quad \sum_{i=1}^n \ell (y_i(\x_i^\top \w + b)) + {\lambda \over 2} \|\w\|^2_A.
\label{eq:classification}
\end{equation}
LB3 will prove to be instrumental in computing the teaching dimension for those learners.
Following standard notation, we define $\btheta=[\w; b]$ where $\w \in \R^d$ is the weight vector and $b \in \R$ the bias (offset) term.
Note $\btheta \in \R^{d+1}$ now.
The $d \times d$ regularization matrix $A$ applies only to $\w$ while $b$ is not regularized.
Furthermore, margin-based linear learners have loss functions defined on the margin $y(\x^\top \w + b)$.
This loss function structure will play a key role in obtaining LB3.

\begin{theorem} \label{thm:classification}
Assume matrix $A$ in \eqref{eq:classification} has full rank and $\w^* \neq {\bf 0}$. Given any target model $[\w^*; b^*]$, there is an inhomogeneous-margin lower bound on the required number of training items to obtain a unique solution $[\w^*; b^*]$ from solving \eqref{eq:classification}:
\begin{align}
LB3 =  \left\lceil {\lambda}{\left(\sup_{\alpha\in \R, g\in -\partial \ell(\alpha \|\w^*\|^2_A)} \alpha g \right)^{-1}} \right\rceil.
\label{eq:thm:classification}
\end{align}
\end{theorem} 


\begin{proof}
Let $D=\{\x_i, y_i\}_{i=1, \cdots, n}$ be a teaching set for $[\w^*; b^*]$. The following KKT condition needs to be satisfied:
\begin{align}
{\bf 0} \in \sum_{i=1}^n \partial \ell (y_i (\x_i^\top \w^* + b^*)) y_i 
\left[
\begin{matrix}
\x_i \\ 1
\end{matrix}
\right]
+ 
\left[
\begin{matrix}
\lambda A \w^* \\ 0
\end{matrix}
\right].
\label{eq:proof:class:1}
\end{align}
If we construct a new training set 
\[
\hat D = \left\{\hat{\x}_i = \x_i + \frac{b^*}{\|\w^*\|^2_A}A\w^*,~\hat{y}_i = y_i \right\}_{i=1, \cdots, n}
\]
then $[\w^*; 0]$ satisfies the KKT condition defined on $\hat D$.
This can be verified as follows:
\begin{align}
\nonumber 
&\sum_{i=1}^n \partial \ell (\hat{y}_i (\hat{\x}_i^\top {\w}^*)) \hat{y_i} 
\left[
\begin{matrix}
\hat{\x}_i \\ 1
\end{matrix}
\right]
+\left[
\begin{matrix}
\lambda A \w^* \\ 0
\end{matrix}
\right]
\\ \nonumber = &
\sum_{i=1}^n \partial \ell (y_i (\x_i^\top \w^* + b^*)) y_i 
\left[
\begin{matrix}
\x_i + \frac{b^*}{\|\w^*\|^2_A}A\w^* \\ 1
\end{matrix}
\right]
+ 
\left[
\begin{matrix}
\lambda A \w^* \\ 0
\end{matrix}
\right]
\\ \nonumber = &
\underbrace{\sum_{i=1}^n \partial \ell (y_i (\x_i^\top \w^* + b^*)) y_i 
\left[
\begin{matrix}
\x_i\\ 1
\end{matrix}
\right]
+ 
\left[
\begin{matrix}
\lambda A \w^* \\ 0
\end{matrix}
\right]}_{\ni {\bf 0}~\text{from~\eqref{eq:proof:class:1}}} + 
\left[
\begin{matrix}
 \frac{b^*}{\|\w^*\|^2_A}A\w^* \\ 0
\end{matrix}
\right]\underbrace{\sum_{i=1}^n \partial \ell (y_i (\x_i^\top \w^* + b^*)) y_i}_{\ni {0} ~\text{from~\eqref{eq:proof:class:1}}}
\\ \ni & {\bf 0} 
\label{eq:proof:class:2}
\end{align}
where $0\in \sum_{i=1}^n \partial \ell (y_i (\x_i^\top \w^* + b^*)) y_i$ is from the bias dimension in \eqref{eq:proof:class:1}. It follows that

\begin{align*}
{\bf 0} \in & \sum_{i=1}^n \partial \ell (\hat{y}_i\hat{\x}_i^\top \w^*) \hat{y}_i \hat{\x}_i + \lambda A\w^*
\end{align*}
which is equivalent to
\begin{align}
\nonumber
{\bf 0} \in & \sum_{i=1}^n \partial \ell (\hat{y}_i\hat{\x}_i^\top \w^*) A^{-{1\over 2}}\underbrace{\hat{y}_i \hat{\x}_i}_{=:\z_i} + \lambda A^{{1\over 2}}\w^*
\\ = &  \sum_{i=1}^n \partial \ell (\z_i^\top \w^*) A^{-{1\over 2}}{\z}_i + \lambda A^{1\over 2}\w^*.
\label{eq:proof:class:3}
\end{align}
We decompose $A^{-{1\over 2}}\z_i = \alpha_i A^{1\over 2} \w^* + \u_i$ where $\u_i$ satisfies $\u_i^\top A^{1\over 2} \w^*=0$. Applying this decomposition to \eqref{eq:proof:class:3}, we have
\begin{align}
\lambda A^{1\over 2}\w^* \in \sum_{i=1}^n - \partial \ell (\alpha_i \|\w^*\|^2_A) (\alpha_i A^{1\over 2} \w^* + \u_i).
\label{eq:proof:class:4}
\end{align}
Since $\u_i$ is orthogonal to $A^{1\over 2}\w^*$, \eqref{eq:proof:class:4} implies that
\begin{align*}
\lambda A^{1\over 2}\w^* \in \sum_{i=1}^n - \partial \ell (\alpha_i \|\w^*\|^2_A) \alpha_i A^{1\over 2} \w^*
\end{align*}
Since $\w^* \neq {\bf 0}$ we have
\[
\lambda \in \sum_{i=1}^n - \partial \ell (\alpha_i \|\w^*\|^2_A) \alpha_i 
\]
Together with 
\[
\sum_{i=1}^n - \partial \ell (\alpha_i \|\w^*\|^2_A) \alpha_i \leq n \sup_{\alpha\in\R, g\in -\partial \ell(\alpha \|\w^*\|^2_A)}\alpha g,
\]
we obtain LB3.
\end{proof}

\subsection{The Teaching Dimension $TD(\btheta^*, \A_{opt})$ of Three Homogeneous Learners}
\label{sec:UBhomogeneous}

We now turn to upper bounding teaching dimension by constructing teaching sets.  
To prove that we indeed have a teaching set for a target $\btheta^*$, we need to show that $\btheta^*$ is a solution of~\eqref{eq:main}, and the solution is unique.
The size of any such teaching set is an upper bound on the teaching dimension.
The teaching dimension itself is determined if such an upper bound matches the corresponding lower bound. 
We show that this is indeed the case for our constructed teaching sets.
For the sake of reference we preview in Table~\ref{tab:LB} the instantiated lower bounds that we will use in this section; their derivation will be shown below.

\begin{table}[ht]
\begin{center}
\begin{tabular}{|c||c|c|c||c|c|c|}
\hline
& \multicolumn{3}{c||}{homogeneous} & \multicolumn{3}{c|}{inhomogeneous} \\ \cline{2-7}
lower bound	& ridge & SVM & logistic & ridge & SVM & logistic \\
\hline
LB1
	& 1
	& 1
	& 1
	& 2
	& 2
	& 2  \\
\hline
LB2
	& 0 
	& $\left\lceil \lambda \|\btheta^*\|^2 \right\rceil$
	& $\left\lceil {\lambda \|\btheta^*\|^2 \over \tau_{\max}} \right\rceil$
	& 0 
	& 0 
	& 0 \\
\hline
LB3
	& -
	& -
	& -
	& -
	& $\left\lceil \lambda \|\w^*\|^2 \right\rceil$
	& $\left\lceil {\lambda \|\w^*\|^2 \over \tau_{\max}} \right\rceil$ \\
\hline
\end{tabular}
\end{center}
\caption{Lower bounds of teaching dimension $TD(\btheta^*, \A_{opt})$ for homogeneous and inhomogeneous versions of ridge regression, SVM, and logistic regression.}  
\label{tab:LB}
\end{table}

Teaching dimension is learner-dependent.
We choose three learners to study their teaching dimension due to these learners' popularity in machine learning: ridge regression, SVM, and logistic regression.  
It turns out that homogeneous and inhomogeneous versions of these learners require different analysis. 
We devote this section to the homogeneous version where the regularizer matrix $A=I$ the identity matrix, and the next section to the inhomogeneous version.
It is possible to extend our analysis to other linear learners of the form~\eqref{eq:main}.

It is easy to see that if the target model $\btheta^*={\bf 0}$, we do not need any training data to uniquely obtain the target model from~\eqref{eq:main}. In the following, we only consider the nontrivial case $\btheta^* \neq {\bf 0}$. 

\textbf{Homogeneous ridge regression} solves the following optimization problem:
\begin{align}
\min_{\btheta\in \R^d}\quad \sum_{i=1}^n {1\over 2} (\x_i^\top \btheta - y_i)^2 + {\lambda \over 2}\|\btheta\|^2
\label{eq:hlr}
\end{align}
We only need one training item to uniquely obtain any nonzero target model $\btheta^*$, as the following construction shows.
\begin{prop} \label{thm:hlr}
Given any target model $\btheta^* \neq 0$, the following is a teaching set for homogeneous ridge regression~\eqref{eq:hlr}:
\begin{equation}
\x_1 = a\btheta^*, \quad
y_1 = \frac{\lambda + \|\x_1\|^2}{a}
\label{eq:hlr:construction}
\end{equation}
where $a$ can be any nonzero real number.
\end{prop}
\begin{proof}
We simply verify the KKT condition to see that $\btheta^*$ is a solution to \eqref{eq:hlr} by applying the construction in \eqref{eq:hlr:construction}. The uniqueness of $\btheta^*$ is guaranteed by the strong convexity of \eqref{eq:hlr}.
\end{proof}

We encourage the reader to distinguish two senses of uniqueness.  The teaching set itself is not necessarily unique.  In the construction~\eqref{eq:hlr:construction}, any $a \neq 0$ leads to a valid teaching set.  
Nonetheless, any one of the teaching sets will lead to the unique solution $\btheta^*$ in~\eqref{eq:hlr}.

\begin{cor}
The teaching dimension $TD(\btheta^*, \A_{ridge}^{hom})=1$ for homogeneous ridge regression and target $\btheta^*\neq \bf 0$. 
\end{cor}
\begin{proof}
Substituting $A$ by $I$ in LB1~\eqref{eq:thm:1}, we obtain the lower bound $d-\text{Rank}(I) + 1 = 1$ which matches the teaching set size in~\eqref{eq:hlr:construction}.
\end{proof}

\textbf{Homogeneous SVM} solves the problem:
\begin{align}
\min_{\btheta\in \R^d}\quad \sum_{i=1}^n \max(1-y_i\x_i^\top \btheta,~0) + {\lambda \over 2}\|\btheta\|^2.
\label{eq:hsvm}
\end{align}
To teach this learner one training item is in general not enough:
we will show that we need $\left\lceil \lambda \|\btheta^*\|^2 \right\rceil$ training items.
In fact, we will construct such a teaching set consisting of \emph{identical} training items.
It is well-known in the teaching literature that a teaching set does not need to consist of $i.i.d.$ samples from a distribution, and can look unusual. 
It is possible to incorporate additional constraints into a teaching problem if one wants the training items to be diverse, but we do not consider that in the present paper.

\begin{prop} \label{thm:hsvm}
Given any target model $\btheta^* \neq 0$, the following is a teaching set for homogeneous SVM~\eqref{eq:hsvm}.  There are
$n=\left\lceil \lambda \|\btheta^*\|^2 \right\rceil$
identical training items, each taking the form
\begin{equation}
\x_i =  \frac{\lambda \btheta^*}{\left\lceil \lambda \|\btheta^*\|^2 \right\rceil}, 
\quad
y_i = 1.
\label{eq:hsvm:construction}
\end{equation}
\end{prop}

\begin{proof}
We only need to verify that the KKT condition holds for $\btheta^*$.
Due to the strong convexity of~\eqref{eq:hsvm} uniqueness is guaranteed automatically.  
We denote the subgradient $\partial_a \max(1-a,0) = - \partial_1 \max(1-a,0) = - {\bf I}(a)$, where
\begin{equation}
{\bf I}(a)= 
\begin{cases}
    1,& \text{if } a<1\\
    [0,1],& \text{if } a=1\\
    0,              & \text{otherwise}
\end{cases}.
\label{eq:def:indictor}
\end{equation}
The KKT condition is
\begin{align*}
& \sum_{i=1}^n -y_i\x_i\partial_1 \max(1- y_i\x_i^\top \btheta^*, 0) + \lambda \btheta^* 
\\ = & 
\sum_{i=1}^n -y_i\x_i \I(y_i\x_i^\top \btheta^*) + \lambda \btheta^* 
\\ = & 
-n\frac{\lambda \btheta^*}{\left\lceil \lambda \|\btheta^*\|^2 \right\rceil} \I\left(\frac{\lambda\|\btheta^*\|^2}{\left\lceil \lambda \|\btheta^*\|^2 \right\rceil}\right) +  \lambda \btheta^*
\\ = &
-\lambda \btheta^* \I\left(\frac{\lambda\|\btheta^*\|^2}{\left\lceil \lambda \|\btheta^*\|^2 \right\rceil}\right) +  \lambda \btheta^* 
\\ \ni & {\bf 0}
\end{align*}
where the last line is due to $\I\left(\frac{\lambda\|\btheta^*\|^2}{\left\lceil \lambda \|\btheta^*\|^2 \right\rceil}\right)$ giving either the set $[0, 1]$ or the value 1.
\end{proof}

\begin{cor}
\label{cor:hsvm}
The teaching dimension $TD(\btheta^*, \A_{svm}^{hom})=\left\lceil \lambda \|\btheta^*\|^2 \right\rceil$ for homogeneous SVM and target $\btheta^*\neq \bf 0$. 
\end{cor}

\begin{proof}
We show this number matches LB2.
Let $A=I$, $\ell(a,b) = \max(1-ab, 0)$, and consider the denominator of~\eqref{eq:thm:3}:
\begin{align}
\nonumber
\sup_{\alpha\in\R, y\in \mathcal{Y}, g\in -\partial_1 \ell(\alpha \|\btheta^*\|^2, y)} \alpha g & = \sup_{\alpha, y\in \{-1, 1\}, g\in y\I(y\alpha \|\btheta^*\|^2)} \alpha g
\\ \nonumber & = 
\sup_{\alpha, g\in \I(\alpha \|\btheta^*\|^2)} \alpha g
\\ & \nonumber
= {1\over \|\btheta^*\|^2}
\end{align}
where the first equality is due to $\partial_1 \ell(a, b) = -b\I(ab)$. Therefore, $LB2=\left\lceil {\lambda \|\btheta^*\|^2}\right\rceil$ which matches the construction in~\eqref{eq:hsvm:construction}.
\end{proof}

\textbf{Homogeneous logistic regression} solves the problem:
\begin{align}
\min_{\btheta\in \R^d}\quad \sum_{i=1}^n\log (1+\exp\{-y_i\x_i^\top \btheta\}) + {\lambda \over 2}\|\btheta\|^2
\label{eq:hlogr}
\end{align}
The situation is similar to homogeneous SVM.  However, due to the negative log likelihood term we have a coefficient defined by the Lambert W function~\cite{Corless1996Lambert}, which we denote by $\W$.
Recall the defining equation for Lambert W function is $\W(x) e^{\W(x)} = x$.
We further define
\begin{equation}
\tau_{\max}:= \max_{t}\frac{t}{1+e^t} = \W(1/e) \approx 0.2785.
\end{equation}
For any value $a\leq \tau_{\max}$, we define $\tau^{-1}(a)$ as the solution to
$a = {t\over 1 + e^t}$.
By using the Lambert W function $\tau^{-1}(a)$ can be expressed as $\tau^{-1}(a) \equiv a - \W(-ae^a)$.

\begin{prop} \label{thm:hlogr}
Given any target model $\btheta^* \neq 0$, the following is a teaching set for homogeneous logistic regression~\eqref{eq:hlogr}.  There are
$n = \left\lceil \frac{\lambda \|\btheta^*\|^2}{\tau_{\max}} \right\rceil$
identical training items, each takes the form
\begin{equation}
\x_i = \tau^{-1}\left(\lambda \|\btheta^*\|^2\left\lceil \frac{\lambda \|\btheta^*\|^2}{\tau_{\max}} \right\rceil^{-1}\right)\frac{\btheta^*}{\|\btheta^*\|^2}, 
\quad
y_i =  1.
\label{eq:hlogr:construction}
\end{equation}
\end{prop}
\begin{proof}
We first verify that $\btheta^*$ is a solution to \eqref{eq:hlogr} based on the teaching set construction in \eqref{eq:hlogr:construction}. We only need to verify the gradient of \eqref{eq:hlogr} is zero. Computing the gradient of \eqref{eq:hlogr}, we have
\begin{align*}
& \sum_{i=1}^n \frac{-y_i\x_i}{1+ \exp\{y_i\x_i^\top \btheta^*\}} + \lambda \btheta^* \\
= & -n \frac{\x_i}{1+\exp\left\{\tau^{-1}\left(\lambda \|\btheta^*\|^2\left\lceil \frac{\lambda \|\btheta^*\|^2}{\tau_{\max}} \right\rceil^{-1}\right)\right\}} + \lambda \btheta^*
\\ = &
-n \frac{\tau^{-1}\left(\lambda \|\btheta^*\|^2\left\lceil \frac{\lambda \|\btheta^*\|^2}{\tau_{\max}} \right\rceil^{-1}\right) }{1+\exp\left\{\tau^{-1}\left(\lambda \|\btheta^*\|^2\left\lceil \frac{\lambda \|\btheta^*\|^2}{\tau_{\max}} \right\rceil^{-1}\right)\right\}}\frac{\btheta^*}{\|\btheta^*\|^2} + \lambda \btheta^*
\\ = &
-n\lambda \|\btheta^*\|^2\left\lceil \frac{\lambda \|\btheta^*\|^2}{\tau_{\max}} \right\rceil^{-1}\frac{\btheta^*}{\|\btheta^*\|^2} + \lambda \btheta^*
\\ = &
{\bf 0},
\end{align*}
where the third equality uses the fact $\lambda \|\btheta^*\|^2\left\lceil \frac{\lambda \|\btheta^*\|^2}{\tau_{\max}} \right\rceil^{-1}\leq \tau_{\max}$ and 
the property $a = \frac{\tau^{-1}(a)}{1+e^{\tau^{-1}(a)}}$. 
The strong convexity of \eqref{eq:hlogr} automatically implies uniqueness. 
\end{proof}

\begin{cor}
\label{cor:hlog}
The teaching dimension $TD(\btheta^*, \A_{log}^{hom})=
\left\lceil \frac{\lambda \|\btheta^*\|^2}{\tau_{\max}} \right\rceil$
for homogeneous logistic regression and target $\btheta^*\neq \bf 0$. 
\end{cor}
\begin{proof}
We show that the number matches LB2.
In~\eqref{eq:thm:3} let $A=I$ and $\ell(a, b)=\log(1+\exp\{-ab\})$. The denominator of LB2 is:
\begin{align*}
\sup_{\alpha\in\R, y\in \mathcal{Y}, g\in - \partial_1 \ell(\alpha \|\btheta^*\|^2, y)} \alpha g = & \sup_{\alpha, y\in \{-1, 1\}, g= y(1+\exp\{y\alpha \|\btheta^*\|^2\})^{-1}} \alpha g
\\ = & 
\sup_{\alpha, g= (1+\exp\{\alpha \|\btheta^*\|^2\})^{-1}} \alpha g
\\ = &
\sup_{\alpha} \frac{\alpha}{1+\exp\{\alpha \|\btheta^*\|^2\}}
\\ = &
\|\btheta^*\|^{-2}\sup_{t} \frac{t}{1+\exp\{t\}}
\\ = &
\frac{\tau_{\max}}{\|\btheta^*\|^2},
\end{align*}
which implies $LB2 = \left\lceil \frac{\lambda \|\btheta^*\|^2}{\tau_{\max}} \right\rceil$.
\end{proof}

\subsection{The Teaching Dimension $TD(\btheta^*, \A_{opt})$ of Three Inhomogeneous Learners}
\label{sec:UBinhomogeneous}

Inhomogeneous learners are defined by $\btheta = [\w; b]$ where the weight vector $\w \in \R^d$ and the scalar offset $b \in \R$.  The offset $b$ is not regularized.
Similar to the previous section, we need to instantiate the teaching dimension lower bounds and design the teaching sets. 
We show that the size of our teaching set exactly matches the lower bound for inhomogeneous ridge regression, and differs from the lower bound of inhomogeneous SVM and logistic regression by at most one due to rounding.
Therefore, up to rounding we also establish the teaching dimension for these inhomogeneous learners.

\textbf{Inhomogeneous ridge regression} solves the problem:
\begin{align}
\min_{\w\in \R^d, b\in \R}\quad \sum_{i=1}^n {1\over 2} (\x_i^\top \w+b - y_i)^2 + {\lambda \over 2}\|\w\|^2
\label{eq:ilr}
\end{align}

\begin{prop} \label{thm:ilr} 
Given any target model $[\w^*; b^*]$, if $\w^*={\bf 0}$ ($b^*$ can be an arbitrary value), the following is a teaching set for inhomogeneous ridge regression~\eqref{eq:ilr} with $n=1$:
\begin{equation}
\x_1 =  {\bf 0},
\quad
y_1 =  b^*.
\label{eq:thm:ilr:1}
\end{equation}
If $\w^*\neq {\bf 0}$, any $n=2$ items satisfying the following are a teaching set for $a \neq 0$:
\begin{equation}
\x_1 - \x_2=  a\w^*,
\quad
y_1 = \x_1^\top \w^* + b^* + \frac{\lambda}{a},
\quad
y_2 =  y_1 - a \|\w^*\|^2 - 2{\lambda \over a}. 
\label{eq:thm:ilr:2}
\end{equation}
\end{prop}
\begin{proof}
We first prove the case for $\w^* = {\bf 0}$. We can verify that the KKT condition is satisfied by designing $\x_1$ and $y_1$ as in~\eqref{eq:thm:ilr:1}:
\begin{align*}
(\x_1^\top \w^* + b^* - y_1)\x_1 + \lambda \w^* = & {\bf 0}
\\
\x_1^\top \w^* + b^* - y_1 = & 0.
\end{align*}
The uniqueness of $[\w^*; b^*]$ is indicated by the strong convexity of \eqref{eq:ilr} when $n=1$.

We then prove the case for $\w^* \neq {\bf 0}$. With simple algebra, we can verify the KKT condition holds via the construction in \eqref{eq:thm:ilr:2}: 
\begin{align*}
(\x_1^\top \w^* + b^* - y_1)\x_1 + (\x_2^\top \w^* + b^* - y_2)\x_2 + \lambda \w^* =&{\bf 0} 
\\
(\x_1^\top \w^* + b^* - y_1) + (\x_2^\top \w^* + b^* - y_2) = &0.
\end{align*}
Similarly, the uniqueness is implied by the strong convexity of \eqref{eq:ilr} when $n=2$. 
\end{proof}

\begin{cor}
\label{cor:iridge}
The teaching dimension 
for inhomogeneous ridge regression with target $\btheta^* = [\w^*; b^*]$ is
$TD(\btheta^*, \A_{ridge}^{inh})=1$ if target $\w^* = \bf 0$, or 
$TD(\btheta^*, \A_{ridge}^{inh})=2$ if $\w^* \neq \bf 0$, regardless of the target offset $b^*$.
\end{cor}
\begin{proof}
We match the lower bound LB1 in~\eqref{eq:thm:1}.
Note $\btheta^*=[\w^*; b^*] \in \R^{d+1}$, and $A$ in this case is a $(d+1)\times (d+1)$ matrix with the $d\times d$ identity matrix $I_d$ padded with one additional row and column of zeros for the offset.
Therefore $Rank(A)=Rank(I_{d})=d$.  When $\w^*=\bf 0$, $A\btheta^*=\bf 0$ and $LB1=(d+1)-Rank(A)=1$.  When $\w^*\neq \bf 0$, $A\btheta^* \neq \bf0$ and $LB1=(d+1)-Rank(A)+1=2$.
These lower bounds match the teaching set sizes in~\eqref{eq:thm:ilr:1} and~\eqref{eq:thm:ilr:2}, respectively.
\end{proof}


\textbf{Inhomogeneous SVM} solves the problem:
\begin{align}
\min_{\w\in \R^d, b\in \R}\quad \sum_{i=1}^n \max(1-y_i(\x_i^\top \w + b),~0) + {\lambda \over 2}\|\w\|^2
\label{eq:isvm}
\end{align}
\begin{prop} \label{thm:isvm}
Given any target model $[\w^*; b^*]$ with $\w^* \neq 0$, the following is a teaching set for inhomogeneous SVM~\eqref{eq:isvm}.
We need 
$n=2\left\lceil {\lambda\|\w^*\|^2 \over 2} \right\rceil$ training items, half of which are identical positive items 
$\x_i = \x_+,\quad y_i = 1,\quad \forall i \in \left\{1, \cdots, {n\over 2} \right\}$ and
half identical negative items
$\x_i = \x_-,\quad y_i = -1,\quad \forall i \in \left\{{n\over 2}+1, \cdots, n \right\}$.
$\x_+$ and $\x_-$ can be designed as any vectors satisfying 
\begin{equation}
\x_+ ^\top \w^*  =  1 - b^* ,
\quad
\x_- =  \x_+ -  \frac{2\w^*} {\|\w^*\|^2}.
\label{eq:isvm:construction}
\end{equation}
\end{prop}
\begin{proof}
Unlike in previous learners (including homogeneous SVM), we no longer have strong convexity w.r.t. $b$.  In order to prove that~\eqref{eq:isvm:construction} is a teaching set, we need to verify the KKT condition and verify solution uniqueness.

We first verify the KKT condition to show that the solution under~\eqref{eq:isvm:construction} includes the target model $[\w^*; b^*]$. From \eqref{eq:isvm:construction}, we have
\begin{equation}
\label{eq:starmargin}
\x_+ ^\top \w^* + b^* =  1, \quad \x_-^\top \w^* + b^* =  -1.
\end{equation}
Applying them to the KKT condition and using the notation in~\eqref{eq:def:indictor}
we obtain
\begin{align*}
& -{n \over 2}\I(\x_+^\top \w^* + b^*)
\left[\begin{matrix}
\x_+ \\
1
\end{matrix}\right]
+
{n \over 2}\I(-\x_-^\top \w^* - b^*)
\left[\begin{matrix}
\x_- \\
1
\end{matrix}\right]
+
\left[\begin{matrix}
\lambda \w^* \\
0
\end{matrix}\right]
\\ = & 
-{n \over 2}\I(1)
\left[\begin{matrix}
\x_+ \\
1
\end{matrix}\right]
+
{n \over 2}\I(1)
\left[\begin{matrix}
\x_- \\
1
\end{matrix}\right]
+
\left[\begin{matrix}
\lambda \w^* \\
0
\end{matrix}\right]
\\
\supset &
{n \over 2}\I(1)
\left[\begin{matrix}
\x_- -\x_+ \\
0
\end{matrix}\right]
+
\left[\begin{matrix}
\lambda \w^* \\
0
\end{matrix}\right] \quad \mbox{setting the last dimension to 0}
\\
= &
\I(1)
\left[\begin{matrix}
-\frac{n}{\|\w^*\|^2}\w^* \\
0
\end{matrix}\right]
+
\left[\begin{matrix}
\lambda \w^* \\
0
\end{matrix}\right] \quad \mbox{applying~\eqref{eq:isvm:construction}}
\\
\supseteq &
\I(1)
\left[\begin{matrix}
-\lambda\w^* \\
0
\end{matrix}\right]
+
\left[\begin{matrix}
\lambda \w^* \\
0
\end{matrix}\right] \quad \mbox{observing $n\geq \lambda \|\w^*\|^2$}
\\
\ni &
{\bf 0}.
\end{align*}
It proves that $[\w^*; b^*]$ solves \eqref{eq:isvm} by our teaching set construction.

Next we prove uniqueness by contradiction. We use $f(\w, b)$ to denote the objective function in \eqref{eq:isvm} under the teaching set.
It is easy to verify that $f(\w^*, b^*) = {\lambda \over 2}\|\w^*\|^2$.
 Assume that there exists another solution $[\bar\w; \bar{b}]$ different from $[\w^*; b^*]$. We can obtain $\|\bar{\w}\|^2 \leq \|\w^*\|^2$ due to
\[
{\lambda \over 2}\|\w^*\|^2 = f(\w^*, b^*) = f(\bar{\w}, \bar{b}) \geq {\lambda \over 2}\|\bar{\w}\|^2.
\]
The second equality is due to $[\bar\w; \bar{b}]$ being a solution; the inequality is due to whole-part relationship.
Therefore, there are only two possibilities for the norm of $\bar{\w}$: $\|\bar{\w}\| = \|\w^*\|$ or $\|\bar{\w}\| = t\|\w^*\|$ for some $0\leq t<1$. Next we will show that both cases are impossible. 

(Case 1) For the case $\|\bar{\w}\| = \|\w^*\|$, we have
\begin{align*}
\nonumber
f(\bar{\w}, \bar{b}) = & {n\over 2} \max \left(1-(\x_+^\top \bar{\w} + \bar{b}), 0 \right) + {n\over 2} \max \left(1+(\x_-^\top \bar{\w} + \bar{b}), 0 \right) + {\lambda \over 2}\|\bar{\w}\|^2 
\\ \nonumber = & 
{n\over 2} \max \left(\underbrace{\x_+^\top (\w^*-\bar{\w}) + (b^*-\bar{b})}_{=:\Delta_+}, 0 \right) + {n\over 2} \max \left(\underbrace{-\x_-^\top (\w^*-\bar{\w}) - (b^*-\bar{b})}_{=:\Delta_-}, 0 \right) 
\\ & 
+ {\lambda \over 2}\|\w^*\|^2
\label{eq:proof:isvm:1}
\\ \nonumber = &
{n\over 2} \max \left(\Delta_+, 0 \right) + {n\over 2} \max \left(\Delta_-, 0 \right) + f(\w^*, b^*).
\end{align*}
From $f(\bar{\w}, \bar{b}) = f(\w^*, b^*)$, it follows $\Delta_+\leq 0$ and $\Delta_-\leq 0$. Since 
\[
 0 \ge \Delta_+ + \Delta_- = (\x_+-\x_-)^\top (\w^* - \bar{\w}) = \frac{2(\w^*)^\top (\w^* - \bar{\w})}{\|\w^*\|^2} = 2 - 2\frac{\bar{\w}^\top \w^*}{\|\w^*\|^2},
\]
we have ${\bar{\w}^\top \w^*}\geq {\|\w^*\|^2}$. 
But because $\|\bar{\w}\| = \|\w^*\|$, we must have $\bar{\w} = \w^*$. 
Applying this new observation to $\Delta_+\leq 0$ and $\Delta_- \leq 0$, we obtain $b^* = \bar{b}$. It means that $[\w^*; b^*] = [\bar{\w}; \bar{b}]$, contradicting our assumption $[\w^*; b^*] \neq [\bar{\w}; \bar{b}]$.

(Case 2) 
Next we turn to the case $\|\bar{\w}\| = t\|\w^*\|$ for some $t\in [0, 1)$.  Recall our assumption that $[\bar{\w}; \bar{b}]$ solves \eqref{eq:isvm}. Then it follows that the following specific construction $[\hat{\w}, \hat{b}]$ solves \eqref{eq:isvm} as well:
\begin{equation}
\label{eq:tstar}
\hat{\w} = t \w^*,\quad \hat{b} = t b^*.
\end{equation}

To see this, we consider the following optimization problem:
\begin{equation}
\begin{aligned}
\min_{\w, b}\quad & L(\w, b):= {n \over 2} \max(1-(\x_+^\top \w+b) ,0) + {n \over 2} \max(1+(\x_-^\top \w+b) ,0) 
\\
\text{s.t.}\quad & \|\w\| \leq t\|\w^*\|.
\end{aligned}
\label{eq:proof:isvm:2}
\end{equation}
Since $[\bar{\w}; \bar{b}]$ solves \eqref{eq:isvm}, it is easy to see that $[\bar{\w}; \bar{b}]$ solves \eqref{eq:proof:isvm:2} too, otherwise there exists a solution for \eqref{eq:proof:isvm:2} which gives a lower function value on \eqref{eq:isvm}. Then we can verify that $[\hat{\w}; \hat{b}]$ solves \eqref{eq:proof:isvm:2} as well by showing the following optimality condition holds:
\begin{equation}
\left.
-\left[
\begin{matrix}
\frac{\partial L(\w, b)}{\partial \w}
\\
\frac{\partial L(\w, b)}{\partial b}
\end{matrix}
\right] 
\right\vert_{[\hat{\w}; \hat{b}]}
\cap 
\underbrace{\mathcal{N}_{\|\w\|\leq t\|\w^*\|}(\hat{\w}, \hat{b})}_{\text{Normal cone to the set $\{[\w;b]: \|\w\|\leq t\|\w^*\|\}$ at $[\hat{\w}; \hat{b}]$}}
\neq \emptyset
\end{equation}
Because of~\eqref{eq:starmargin} and~\eqref{eq:tstar}, we have $\x_+^\top \hat\w + \hat b = t < 1$.
Thus at $[\hat \w; \hat b]$ the subgradient is 
\begin{equation}
\left.
- \left[
\begin{matrix}
\frac{\partial L(\w, b)}{\partial \w}
\\
\frac{\partial L(\w, b)}{\partial b}
\end{matrix}
\right] 
\right\vert_{[\hat{\w}; \hat{b}]}
=
{n\over 2}\left[
\begin{matrix}
\x_+ - \x_-
\\
{0}
\end{matrix}
\right]
=
\left[
\begin{matrix}
\frac{n\w^*}{\|\w^*\|^2}
\\
{0}
\end{matrix}
\right]
\end{equation}
And the normal cone is
\begin{equation}
\mathcal{N}_{\|\w\|\leq t\|\w^*\|}(\hat\w,\hat{b}) = \left\{s \left[
\begin{matrix}
\w^*
\\
{0}
\end{matrix}
\right]~\Bigg|~s\geq 0 \right\}.
\end{equation}
The intersection is non-empty by choosing $s = {n \over \|\w^*\|^2}$.
Since both $[\hat{\w}; \hat{b}]$ and $[\bar{\w}; \bar{b}]$ solve \eqref{eq:proof:isvm:2}, we have $L(\hat{\w}, \hat{b}) = L(\bar{\w}, \bar{b})$. Together with $\|\hat{\w}\| = \|\bar{\w}\|$, we have 
\[
f(\hat{\w}, \hat{b}) = L(\hat{\w}, \hat{b}) + {\lambda \over 2}\|\hat{\w}\|^2 = f(\bar{\w}, \bar{b}) = f(\w^*, b^*).
\] 
Therefore, we proved that $[\hat{\w}; \hat{b}]$ solves \eqref{eq:isvm} as well. 
To see the contradiction, let us check the function value of $f(\hat{\w}, \hat{b})$ via a different route:
\begin{align*}
f(\hat{\w}, \hat{b}) = & f(t\w^*, tb^*) 
\\ = & 
\sum_{i=1}^{n\over 2} \max \left( 1-t(\x_+^\top \w^* + b^*),0 \right) + \sum_{i=1}^{n\over 2} \max \left( 1+t(\x_-^\top \w^* + b^*), 0 \right) + {\lambda \over 2}\|\w^*\|^2 t^2
\\ = & 
\sum_{i=1}^{n\over 2} \max \left( 1-t, 0 \right) + \sum_{i=1}^{n\over 2} \max \left( 1-t,0 \right) + {\lambda \over 2}\|\w^*\|^2 t^2
\\ = &
n (1-t) - {\lambda \over 2} \|\w^*\|^2 (1-t^2) + {\lambda \over 2} \|\w^*\|^2
\\ \geq &
n (1-t) - {n \over 2} (1-t^2) + {\lambda \over 2} \|\w^*\|^2
\\ = &
{n\over 2} (1-t)^2 + f(\w^*, b^*)
\\ > &
f(\w^*, b^*),
\end{align*}
where the first inequality uses the fact that $n \geq \lambda \|\w^*\|^2$. It contradicts our early assertion $f(\hat{\w}, \hat{b}) = f(\w^*, b^*) $. 
Putting cases 1 and 2 together we prove uniqueness.
\end{proof}

Our construction of the teaching set in~\eqref{eq:isvm:construction} requires 
$n=2\left\lceil {\lambda\|\w^*\|^2 \over 2} \right\rceil$ training items.
This is an upper bound on the teaching dimension.
Meanwhile, we show below that the inhomogeneous SVM lower bound is $LB3=\left\lceil \lambda\|\w^*\|^2 \right\rceil$.
There can be a difference of at most one between the lower and upper bounds, which we call the ``rounding effect.''
We suspect that this small gap is a technicality and not intrinsic.  However, at present we do not have a teaching set construction that bridges this gap.
Therefore, we state the teaching dimension as an interval in the following corollary and leave the precise value as an open question for future research.

\begin{cor}
The teaching dimension 
for inhomogeneous SVM and target $\btheta^* = [\w^*; b^*]$ where $\w^*\neq \bf 0$ is in the interval 
$\left\lceil \lambda\|\w^*\|^2 \right\rceil \le
TD(\btheta^*, \A_{svm}^{inh})
\le 2\left\lceil {\lambda\|\w^*\|^2 \over 2} \right\rceil$.
\end{cor}

\begin{proof}
The upper bound directly follows Proposition~\ref{thm:isvm}. We only need to show the lower bound $LB3=\left\lceil \lambda \|\w^*\|^2\right\rceil$ in Theorem~\ref{thm:classification}. Let $A=I$, $\ell(a) = \max(1-a, 0)$, and consider the denominator of~\eqref{eq:thm:classification}:
$$\sup_{\alpha\in\R, g\in -\partial \ell(\alpha \|\w^*\|^2)} \alpha g  = \sup_{\alpha, g\in \I(\alpha \|\w^*\|^2)} \alpha g = {1\over \|\w^*\|^2}$$
where the first equality is due to $\partial \ell(a) = -\I(a)$. Therefore, $LB3=\left\lceil {\lambda \|\w^*\|^2}\right\rceil$ which proves the lower bound.
\end{proof}

\textbf{Inhomogeneous logistic regression} solves the problem
\begin{align}
\min_{\w\in \R^d, b\in \R}\quad \sum_{i=1}^n\log (1+\exp\{-y_i(\x_i^\top \w+b)\}) + {\lambda \over 2}\|\w\|^2
\label{eq:ilogr}
\end{align}

\begin{prop} \label{thm:ilogr}
To create a teaching set for target model $[\w^*; b^*]$ with nonzero $\w^*$ for inhomogeneous logistic regression~\eqref{eq:ilogr}, we can use 
$n = 2\left\lceil {\lambda \|\w^*\|^2 \over {2\tau_{\max}}} \right\rceil$
training items where
$\x_i = \x_+,\quad y_i = 1,\quad \forall i \in \left\{1, \cdots, {n\over 2} \right\}$
and
$\x_i = \x_-,\quad y_i = -1,\quad \forall i \in \left\{{n\over 2}+1, \cdots, n \right\}$.
$\x_+$ and $\x_-$ can be designed as any vectors satisfying 
\begin{equation}
\x_+ ^\top \w^*  =  t - b^* ,
\quad
\x_- =  \x_+ -  \frac{2t}{\|\w^*\|^2}\w^*,
\label{eq:ilogr:construction}
\end{equation}
where the constant $t$ is defined by
$t:= \tau^{-1}\left(\frac{\lambda \|\w^*\|^2}{n}\right)$.
\end{prop}

\begin{proof}
We first point out that for $t$ to be well-defined the argument to $\tau^{-1}()$ has to be bounded $\frac{\lambda \|\w^*\|^2}{n} \le \tau_{\max}$.
This implies $n \ge \frac{\lambda \|\w^*\|^2}{\tau_{\max}}$.
The size of our proposed teaching set is the smallest among all such symmetric construction that satisfy this constraint.

We verify that the KKT condition to show the construction in~\eqref{eq:ilogr:construction} includes the solution $[\w^*; b^*]$. From \eqref{eq:ilogr:construction}, we have
\[
\x_+^\top \w^* + b^* = t\quad \x_-^\top \w^* + b^* = -t.
\]
We apply them and the teaching set construction to compute the gradient of~\eqref{eq:ilogr}:
\begin{align*}
& -{n \over 2}\frac{1}{1+ \exp\{\x_+^\top \w^* + b^*\}}
\left[\begin{matrix}
\x_+ \\
1
\end{matrix}\right]
+
{n \over 2}\frac{1}{1+ \exp\{-\x_-^\top \w^* - b^*\}}
\left[\begin{matrix}
\x_- \\
1
\end{matrix}\right]
+
\left[\begin{matrix}
\lambda \w^* \\
0
\end{matrix}\right]
\\ = & 
-{n \over 2}\frac{1}{1+ \exp\{t\}}
\left[\begin{matrix}
\x_+ \\
1
\end{matrix}\right]
+
{n \over 2}\frac{1}{1+ \exp\{t\}}
\left[\begin{matrix}
\x_- \\
1
\end{matrix}\right]
+
\left[\begin{matrix}
\lambda \w^* \\
0
\end{matrix}\right]
\\ = & 
-{n\over \|\w^*\|^2}\frac{t}{1+ \exp\{t\}}
\left[\begin{matrix}
\w^*
\\
0
\end{matrix}\right]
+
\left[\begin{matrix}
\lambda \w^* \\
0
\end{matrix}\right]
\\ = &
-{n\over \|\w^*\|^2}\frac{\lambda \|\w^*\|^2}{n}
\left[\begin{matrix}
\w^*
\\
0
\end{matrix}\right]
+
\left[\begin{matrix}
\lambda \w^* \\
0
\end{matrix}\right]
\\ = &
{\bf 0}.
\end{align*}
This verifies the KKT condition. 

Finally we show uniqueness. The Hessian matrix of the objective function~\eqref{eq:ilogr} under our training set~\eqref{eq:ilogr:construction} is:
\begin{align*}
\underbrace{{n \over 2} \frac{\exp\{t\}}{(1+\exp \{t\})^2}}_{:=a}
\underbrace{\left[
\begin{matrix}
\x_+ \x_+^\top + \x_- \x_-^\top  & \x_+ + \x_- \\
\x_+^\top + \x_-^\top & 2
\end{matrix}
\right]
}_{{:=A}}
+ \lambda 
\underbrace{
\left[
\begin{matrix}
I & {\bf 0} \\
{\bf 0}^\top & 0
\end{matrix}
\right]
}_{:=B}.
\end{align*}
Note $a>0$ and $A=
\begin{bmatrix} \x_+ \\ 1 \end{bmatrix} \begin{bmatrix} \x_+ & 1 \end{bmatrix}
+
\begin{bmatrix} \x_- \\ 1 \end{bmatrix} \begin{bmatrix} \x_- & 1 \end{bmatrix}
$ is positive semi-definite.
We show that $aA+\lambda B$ is positive definite.
Suppose not.  Then there exists $[{\bf u}; v] \neq \bf 0$ such that 
$[{\bf u}; v]^\top (aA+\lambda B) [{\bf u}; v] = 0$.
This implies 
$[{\bf u}; v]^\top (aA) [{\bf u}; v] + \lambda {\bf u}^\top {\bf u} = 0$.
Since the first term is non-negative due to $A$ being positive semi-definite, ${\bf u} = \bf 0$.
But then we have $2a v^2 = 0$ which implies $[{\bf u}; v] = \bf 0$, a contradiction.
Therefore uniqueness is guaranteed. 

\end{proof}

\begin{cor}
The teaching dimension 
for inhomogeneous logistic regression and target $\btheta^* = [\w^*; b^*]$ where $\w^*\neq \bf 0$ is in the interval 
$\left\lceil {\lambda\|\w^*\|^2 \over \tau_{\max}} \right\rceil \le
TD(\btheta^*, \A_{log}^{inh})
\le 2\left\lceil {\lambda\|\w^*\|^2 \over {2 \tau_{\max} }} \right\rceil$.
\end{cor}
\begin{proof}
The upper bound directly follows Proposition~\ref{thm:ilogr}. We only need to show the lower bound $\left\lceil {\lambda \|\w^*\|^2 \over \tau_{\max}}\right\rceil$ by applying $LB3$ in Theorem~\ref{thm:classification}. Let $A=I$ and $\ell(a)=\log(1+\exp\{-a\})$ and consider the denominator of~\eqref{eq:thm:classification}:
\begin{align*}
\sup_{\alpha\in\R, g\in \partial \ell(-\alpha \|\w^*\|^2)} \alpha g  = & 
\sup_{\alpha, g= (1+\exp\{\alpha \|\w^*\|^2\})^{-1}} \alpha g
\\ = &
\sup_{\alpha} \frac{\alpha}{1+\exp\{\alpha \|\w^*\|^2\}}
\\ = &
\|\w^*\|^{-2}\sup_{t} \frac{t}{1+\exp\{t\}}
\\ = &
\frac{\tau_{\max}}{\|\w^*\|^2},
\end{align*}
which implies $LB3 = \left\lceil \frac{\lambda \|\w^*\|^2}{\tau_{\max}} \right\rceil$.
\end{proof}

\section{Teaching a Decision Boundary Instead of a Parameter} 
\label{sec:Gdb}
In section~\ref{sec:main} we considered the teaching goal where the learner is required to learn the exact \emph{target parameter} $\btheta^*$.
But when the learner is a classifier often a weaker teaching goal is sufficient, namely teaching the learner a \emph{target decision boundary}. 
In this section we consider this teaching goal.
Equivalently, such a goal is defined by the set of parameters that produce the target decision boundary.
Teaching is successful if the learner arrives at any one parameter within that set.

In the case of inhomogeneous linear learners, the linear decision boundary 
$\{\x \mid \x^\top \w^* + b^* = 0\}$
is identified with the parameter set
$\{t[\w^*; b^*]: t>0\}$.  Here we assume $\w^*$ is nonzero.
The parameter $\btheta^* = [\w^*; b^*]$ is just a representative member of the set.
Homogeneous linear learners are similar without $b^*$.
We denote the corresponding ``decision boundary'' teaching dimension by $TD(\{ t \btheta^*\},  \A_{opt})$. 
This notation extends our earlier definition of TD by allowing the first argument to be a set, with the understanding that the teaching goal is for the learned model to be an element in the set.
It immediately follows that 
\begin{equation}
TD(\{ t \btheta^*\},  \A_{opt}) = \min_{t>0} TD(t \btheta^*,  \A_{opt}). 
\end{equation}
Since it is sufficient to teach the parameter $t\btheta^*$ for some $t>0$ in order to teach the decision boundary,
we can choose the best $t$ that minimizes $TD( t \btheta^*,  \A_{opt})$.
For SVM and logistic regression -- either homogeneous or inhomogeneous -- the teaching dimension $TD(t \btheta^*,  \A_{opt})$ depends on $\|t\btheta^*\|$ (see Table~\ref{tab:main}).
We can choose $t$ sufficiently small to drive down the teaching set size toward its minimum (which is nonzero because of the ceiling function).  
Specifically, for any fixed parameter $\btheta^*$ representing the target decision boundary:
\begin{itemize}
\item (homogeneous SVM): we can choose $t \le \frac{1}{ \sqrt{\lambda} \|\btheta^*\|}$ so that $TD(\{ t \btheta^*\},  \A^{hom}_{svm})=1$;
\item (homogeneous logistic regression): we can choose $t \le \frac{\sqrt{\tau_{\max}}}{ \sqrt{\lambda} \|\btheta^*\|}$ so that $TD(\{ t \btheta^*\},  \A^{hom}_{log})=1$;
\item (inhomogeneous SVM): we can choose $t \le \frac{\sqrt{2}}{ \sqrt{\lambda} \|\w^*\|}$ so that $TD(\{ t \btheta^*\},  \A^{inh}_{svm})=2$;
\item (inhomogeneous logistic regression): we can choose $t \le \frac{\sqrt{2\tau_{\max}}}{ \sqrt{\lambda} \|\w^*\|}$ so that $TD(\{ t \btheta^*\},  \A^{inh}_{log})=2$.
\end{itemize}
The resulting teaching dimension $TD(\{ t \btheta^*\},  \A_{opt})$ is listed in Table~\ref{tab:main} on the row marked by ``decision boundary.''
The teaching set construction is the same as in sections~\ref{sec:UBhomogeneous} and~\ref{sec:UBinhomogeneous}, respectively, but with $t\btheta^*$.

\section{Related Work}
\label{sec:relatedwork}

Teaching dimension as a learning-theoretic quantity has attracted a long history of research.
It was proposed independently in~\cite{Goldman1995Complexity,Shinohara1991Teachability}. 
Subsequent theoretical developments can be found in e.g.~\cite{Zilles2011Models,Balbach2009Recent,982362,conf/colt/AngluinK97,Goldman1996Teaching,DBLP:journals/jcss/Mathias97,Balbach2006Teaching,Balbach:2008:MTU:1365093.1365255,Kobayashi2009Complexity,journals/ml/AngluinK03,conf/colt/RivestY95,journals/ml/Ben-DavidE98,JMLR:v15:doliwa14a}.
Most of them assume little extra knowledge on the learner other than that it is consistent with the training data.
While such version-space learners are elegant object of theoretical study,
they diverge from the practice of modern machine learning.
Our present paper is among the first to extend teaching dimension to optimization-based machine learners.

Teaching dimension is distinct from VC dimension.
For a finite hypothesis space $\H$,
Goldman and Kearns~\cite{Goldman1995Complexity} proved the relation
\begin{equation}
VC(\H) / \log(|\H|) \le TD(\H) \le VC(\H) + |\H| - 2^{VC(\H)}. 
\end{equation}
These inequalities are somewhat weak, as Goldman and Kearns had shown both cases where one quantity is much larger than the other.
The distinction between TD and VC dimension is also present in our setting.
For example, by inspecting the inhomogeneous SVM column in Table~\ref{tab:main} we note that TD does not depend on the dimensionality $d$ of the feature space $\R^d$.
To see why this makes intuitive sense, note two $d$-dimensional points are sufficient to specify any bisecting hyperplane in $\R^d$.
On the other hand, recall that the VC dimension for inhomogeneous hyperplanes in $\R^d$ is $d+1$.
Further quantification of the relation between TD and VC (and other capacity measures) remains an open research question.

The teaching setting we considered is also distinct from active learning.
In teaching the teacher knows the target model \emph{a priori} and her goal is to \emph{encode} the target model as a training set, 
knowing that the decoder is special (namely a specific machine learning algorithm). 
This communication perspective highlights the difference to active learning, which must explore the hypothesis space to find the target model.
 Consequently, the teaching dimension can be dramatically smaller than the active learning query complexity for the same learner and hypothesis space.
For example, Zhu~\cite{Zhu2013Machine} demonstrated that to learn a 1D threshold classifier within $\epsilon$ error, the teaching dimension is a constant TD=2 regardless of $\epsilon$, while active learning would require $O(\log{1 \over \epsilon})$ queries which can be arbitrarily larger than TD.

While the present paper focused on the theory of optimal teaching, there are practical applications, too.
One such application is computer-aided personalized education.
The human student is modeled by a computational cognitive model, or equivalently the learning algorithm.
The educational goal is encoded in the target model. 
The optimal teaching set is then well-defined, and represents the best personalized lesson for the student~\cite{Zhu2015Machine,Zhu2013Machine,Khan2011How}.
Patil \textit{et al.} showed that human students learn statistically significantly better under such optimal teaching set compared to an $i.i.d.$ training set~\cite{Patil2014Optimal}.
Because contemporary cognitive models often employ optimization-based machine learners,
our teaching dimension study helps to characterize these optimal lessons.

Another application of optimal teaching is in computer security.
In particular, optimal teaching is the mathematical formalism to study the so-called data poisoning attacks~\cite{Barreon-Nelson-MLJ-2010,Mei2015Machine,Mei2015Security,Alfeld2016Data}.
Here the ``teacher'' is an attacker who has a nefarious target model in mind.  
The ``student'' is a learning agent (such as a spam filter) which accepts data and adapts itself.
The attacker wants to minimally manipulate the input data in order to manipulate the learning agent toward the attacker's target model.
Teaching dimension quantifies the difficulty of data-poisoning attacks, and enables research on defenses.


Teaching dimension also has applications in interactive machine learning to quantify the minimum human interaction necessary~\cite{Cakmak2011Mixed}, and in formal synthesis to generate computer programs satisfying a specification~\cite{DBLP:journals/corr/JhaS15}.

\section{Conclusion}

We have presented a generalization on teaching dimension to optimization-based learners.
To the best of our knowledge, our teaching dimension for ridge regression, SVM, and logistic regression is new;
so are the lower bounds and our analysis technique in general.

There are many possible extensions to the present work.
For example, one may extend our analysis to nonlinear learners. This can potentially be achieved by using the kernel trick on the linear learners.
As another example, one may allow ``approximate teaching'' by relaxing the teaching goal, such that teaching is considered successful if the learner arrives at a model close enough to the target model.
Taken together, the present paper and its extensions are expected to enrich our understanding of optimal teaching and enable novel applications.

{
\bibliography{reference}

\begin{thebibliography}{10}

\bibitem{Alfeld2016Data}
S.~Alfeld, X.~Zhu, and P.~Barford.
\newblock Data poisoning attacks against autoregressive models.
\newblock {\em AAAI}, 2016.

\bibitem{982362}
D.~Angluin.
\newblock Queries revisited.
\newblock {\em Theoretical Computer Science}, 313(2):175--194, 2004.

\bibitem{conf/colt/AngluinK97}
D.~Angluin and M.~Krikis.
\newblock Teachers, learners and black boxes.
\newblock {\em COLT}, 1997.

\bibitem{journals/ml/AngluinK03}
D.~Angluin and M.~Krikis.
\newblock Learning from different teachers.
\newblock {\em Machine Learning}, 51(2):137--163, 2003.

\bibitem{Balbach:2008:MTU:1365093.1365255}
F.~J. Balbach.
\newblock Measuring teachability using variants of the teaching dimension.
\newblock {\em Theor. Comput. Sci.}, 397(1-3):94--113, May 2008.

\bibitem{Balbach2006Teaching}
F.~J. Balbach and T.~Zeugmann.
\newblock Teaching randomized learners.
\newblock {\em COLT}, pages 229--243, 2006.

\bibitem{Balbach2009Recent}
F.~J. Balbach and T.~Zeugmann.
\newblock Recent developments in algorithmic teaching.
\newblock In {\em Proceedings of the 3rd International Conference on Language
  and Automata Theory and Applications}, pages 1--18, 2009.

\bibitem{Barreon-Nelson-MLJ-2010}
M.~Barreno, B.~Nelson, A.~D. Joseph, and J.~D. Tygar.
\newblock The security of machine learning.
\newblock {\em Machine Learning Journal}, 81(2):121--148, 2010.

\bibitem{journals/ml/Ben-DavidE98}
S.~Ben-David and N.~Eiron.
\newblock Self-directed learning and its relation to the {VC}-dimension and to
  teacher-directed learning.
\newblock {\em Machine Learning}, 33(1):87--104, 1998.

\bibitem{Cakmak2011Mixed}
M.~Cakmak and A.~Thomaz.
\newblock Mixed-initiative active learning.
\newblock {\em ICML Workshop on Combining Learning Strategies to Reduce Label
  Cost}, 2011.

\bibitem{Corless1996Lambert}
R.~Corless, G.~Gonnet, D.~Hare, D.~Jeffrey, and D.~Knuth.
\newblock On the {L}ambert{W} function.
\newblock {\em Advances in Computational Mathematics}, 5(1):329--359, 1996.

\bibitem{JMLR:v15:doliwa14a}
T.~Doliwa, G.~Fan, H.~U. Simon, and S.~Zilles.
\newblock Recursive teaching dimension, {VC}-dimension and sample compression.
\newblock {\em Journal of Machine Learning Research}, 15:3107--3131, 2014.

\bibitem{Goldman1995Complexity}
S.~Goldman and M.~Kearns.
\newblock On the complexity of teaching.
\newblock {\em Journal of Computer and Systems Sciences}, 50(1):20--31, 1995.

\bibitem{Goldman1996Teaching}
S.~Goldman and H.~Mathias.
\newblock Teaching a smarter learner.
\newblock {\em Journal of Computer and Systems Sciences}, 52(2):255–267,
  1996.

\bibitem{DBLP:journals/corr/JhaS15}
S.~Jha and S.~A. Seshia.
\newblock A theory of formal synthesis via inductive learning.
\newblock {\em CoRR}, abs/1505.03953, 2015.

\bibitem{Khan2011How}
F.~Khan, X.~Zhu, and B.~Mutlu.
\newblock How do humans teach: On curriculum learning and teaching dimension.
\newblock {\em NIPS}, 2011.

\bibitem{Kobayashi2009Complexity}
H.~Kobayashi and A.~Shinohara.
\newblock Complexity of teaching by a restricted number of examples.
\newblock In {\em COLT}, pages 293--302, 2009.

\bibitem{DBLP:journals/jcss/Mathias97}
H.~D. Mathias.
\newblock A model of interactive teaching.
\newblock {\em J. Comput. Syst. Sci.}, 54(3):487--501, 1997.

\bibitem{Mei2015Security}
S.~Mei and X.~Zhu.
\newblock The security of latent {D}irichlet allocation.
\newblock {\em AISTATS}, 2015.

\bibitem{Mei2015Machine}
S.~Mei and X.~Zhu.
\newblock Using machine teaching to identify optimal training-set attacks on
  machine learners.
\newblock {\em AAAI}, 2015.

\bibitem{Patil2014Optimal}
K.~Patil, X.~Zhu, L.~Kopec, and B.~Love.
\newblock Optimal teaching for limited-capacity human learners.
\newblock {\em NIPS}, 2014.

\bibitem{conf/colt/RivestY95}
R.~L. Rivest and Y.~L. Yin.
\newblock Being taught can be faster than asking questions.
\newblock {\em COLT}, 1995.

\bibitem{Shinohara1991Teachability}
A.~Shinohara and S.~Miyano.
\newblock Teachability in computational learning.
\newblock {\em New Generation Computing}, 8(4):337–--348, 1991.

\bibitem{Zhu2013Machine}
X.~Zhu.
\newblock Machine teaching for {B}ayesian learners in the exponential family.
\newblock {\em NIPS}, 2013.

\bibitem{Zhu2015Machine}
X.~Zhu.
\newblock Machine teaching: an inverse problem to machine learning and an
  approach toward optimal education.
\newblock {\em AAAI}, 2015.

\bibitem{Zilles2011Models}
S.~Zilles, S.~Lange, R.~Holte, and M.~Zinkevich.
\newblock Models of cooperative teaching and learning.
\newblock {\em Journal of Machine Learning Research}, 12:349--384, 2011.

\end{thebibliography}
\bibliographystyle{abbrv}
}

\end{document}